\documentclass{article}

\usepackage{microtype}
\usepackage{graphicx}
\usepackage{subcaption}
\usepackage{booktabs} 
\usepackage{hyperref}
\usepackage{kotex}
\usepackage{comment}
\usepackage{multirow}
\usepackage{xcolor}
\usepackage{colortbl}

\usepackage[accepted]{icml2026}

\usepackage{amsmath}
\usepackage{amssymb}
\usepackage{mathtools}
\usepackage{amsthm}

\usepackage[capitalize,noabbrev]{cleveref}
\usepackage[textsize=tiny]{todonotes}

\theoremstyle{plain}
\newtheorem{theorem}{Theorem}[section]

\newtheorem{lemma}[theorem]{Lemma}

\theoremstyle{definition}

\newtheorem{assumption}[theorem]{Assumption}
\theoremstyle{remark}

\icmltitlerunning{Efficient Epistemic Uncertainty Estimation for Large Language Models via Knowledge Distillation}

\begin{document}

\makeatletter
\newcounter{@affilgsds}\setcounter{@affilgsds}{1} 
\newcounter{@affilae}\setcounter{@affilae}{2}     
\newcounter{@affilrse}\setcounter{@affilrse}{3}   
\newcounter{@affilgist}\setcounter{@affilgist}{4} 
\setcounter{@affiliationcounter}{4}               
\makeatother

\twocolumn[
  \icmltitle{Efficient Epistemic Uncertainty Estimation\\for Large Language Models via Knowledge Distillation}

  \icmlsetsymbol{corr}{$\dagger$}

    \begin{icmlauthorlist}
      \icmlauthor{Seonghyeon Park}{ae}
      \icmlauthor{Jewon Yeom}{gsds}
      \icmlauthor{Jaewon Sok}{rse}
      \icmlauthor{Jeongjae Park}{gsds}
      \icmlauthor{Heejun Kim}{gist}
      \icmlauthor{Taesup Kim}{gsds,corr}
      
    \end{icmlauthorlist}
    
    \icmlaffiliation{gsds}{Graduate School of Data Science, Seoul National University}
    \icmlaffiliation{rse}{Department of Rural Systems Engineering, Seoul National University}
    \icmlaffiliation{ae}{Department of Aerospace Engineering, Seoul National University}
    \icmlaffiliation{gist}{Department of Electrical Engineering and Computer Science, Gwangju Institute of Science and Technology}
    
    \icmlcorrespondingauthor{Taesup Kim}{taesup.kim@snu.ac.kr}

  \icmlkeywords{Uncertainty Estimation, Large Language Models, Knowledge Distillation, Speculative Decoding}

  \vskip 0.3in

]

\printAffiliationsAndNotice{}

\begin{abstract}
Quantifying uncertainty in Large Language Models (LLMs) is essential for mitigating hallucinations and enabling risk-aware deployment in safety-critical tasks. 
However, estimating \emph{Epistemic Uncertainty} (EU) via Deep Ensembles is computationally prohibitive at the scale of modern models. 
We propose a framework that leverages the small draft models to efficiently estimate token-level EU, bypassing the need for full-scale ensembling.
Theoretically grounded in a \textit{Bias-Variance Decomposition}, our approach approximates EU via Jensen-Shannon divergence among drafts (variance proxy) and KL divergence between the draft mixture and the target (bias proxy). 
To further ensure accuracy without significant overhead, we introduce \textit{Online Stochastic Distillation} (OSD) to efficiently approximate target aggregation and the \textit{Data-Diverse Drafts} (DDD) strategy to enhance draft diversity for better target approximation.
Extensive experiments on GSM8K demonstrate that our method reduces the estimation error (RMSE) by up to 37\% compared to baselines.
Crucially, our approach achieves Hallucination Detection performance competitive with heavy perturbation-based methods like TokUR while incurring negligible inference costs, offering a practical solution for uncertainty-aware LLM deployment.
\end{abstract}

\section{Introduction}
\label{sec:intro}
Large Language Models (LLMs) have demonstrated remarkable performance across a wide range of reasoning, generation, and decision-making tasks \cite{bubeck2023sparks, gemini2023family}.
Despite these advances, their tendency to produce fluent yet incorrect responses, which is commonly referred to as \emph{hallucinations}, remains a critical obstacle to reliable real-world deployment \cite{huang2025surveyHallucinationLLMs, alansari2025hallucinationSurvey, ji2023surveyHallucinationNLG}. 
A key challenge underlying hallucinations is the lack of reliable uncertainty estimation, which prevents models from recognizing when they are likely to be wrong.

\begin{figure}[t]
  \centering
  \includegraphics[width=\columnwidth]{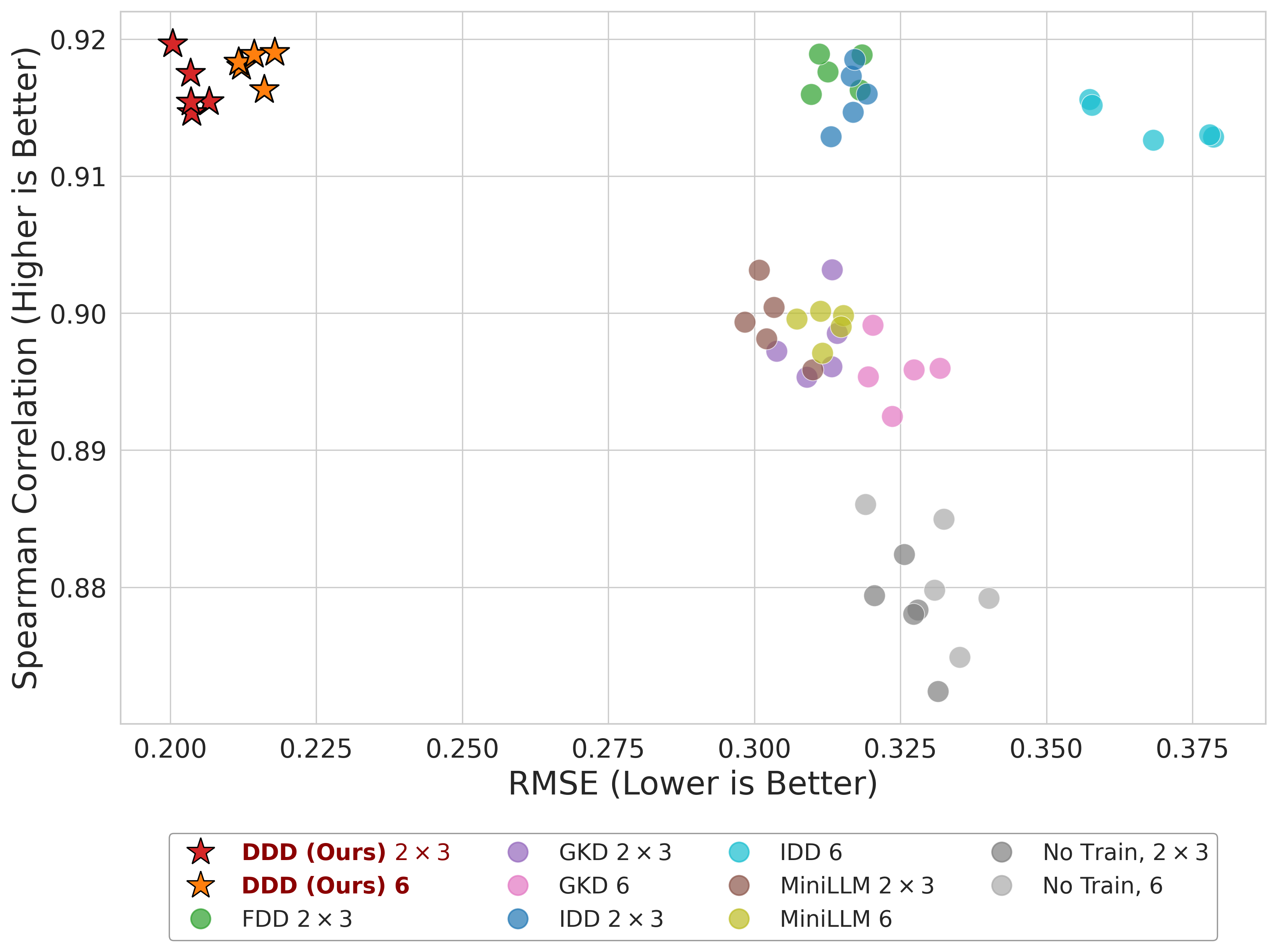}
  \caption{\textbf{RMSE vs. Spearman correlation for 8B (target model) $\rightarrow$ 3B (draft model).} Each point represents one of 5 independent runs. Lower RMSE and higher Spearman indicate better uncertainty estimation. Our proposed DDD achieves the best trade-off among all methods.}
  \label{fig:main_rmse_spearman_scatter}
\end{figure}

Uncertainty estimation in neural networks is commonly framed through a Bayesian lens, decomposing predictive uncertainty into \emph{aleatoric uncertainty}, which arises from inherent noise in the data, and \emph{epistemic uncertainty}, which reflects uncertainty over model parameters \cite{kendall2017uncertainties, hullermeier2021aleatoric}. 
Among these, \emph{epistemic uncertainty} (EU) is particularly valuable for safety-critical applications, as it signals insufficient knowledge and correlates strongly with out-of-distribution (OOD) inputs and hallucination-prone queries~\cite{ovadia2019trustUncertainty, gal2016dropout, farquhar2024semanticEntropyHallucination, zhang2025tokur}.
Unlike aleatoric uncertainty which captures inherent data noise, epistemic uncertainty effectively quantifies the model's `ignorance' stemming from data scarcity. Consequently, high epistemic uncertainty serves as a direct indicator of potential confabulation, making it a vital reliability metric for high-stake domains.
Accurately estimating EU enables models to abstain, or trigger verification mechanisms, thereby significantly improving trustworthiness.

However, estimating epistemic uncertainty for modern LLMs remains challenging. 
Bayesian approaches such as \emph{Deep Ensembles} provide a principled approximation by integrating over multiple independently trained models~\cite{Lakshminarayanan2017}, but their computational and memory costs are prohibitive for models with billions of parameters. 
Alternative perturbation-based methods, such as dropout sampling, input paraphrasing, or token-level resampling, avoid maintaining multiple model checkpoints \cite{gal2016dropout, ovadia2019trustUncertainty,kuhn2023semantic}.
However, in autoregressive generation, these approaches often conflate epistemic and aleatoric uncertainty and require repeated forward passes of the same large model~\cite{ovadia2019trustUncertainty, farquhar2024semanticEntropyHallucination}. 
These limitations significantly restrict their applicability in latency-sensitive or large-scale settings.

At the same time, recent advances in inference acceleration have introduced an overlooked opportunity. 
Efficient inference paradigms~\citep{Leviathan2023} employ one or more lightweight draft models to propose candidate tokens, which are subsequently verified by a larger target model. 
Although originally designed for efficiency, this strategy naturally produces a family of auxiliary models that approximate the target model's predictive behavior while differing in architecture size, training data, and optimization dynamics.

From a Bayesian perspective, epistemic uncertainty arises from variability across plausible models sampled from the posterior over parameters. 
While directly sampling from the posterior of an LLM is infeasible, we recognize that a single lightweight draft model often lacks the expressivity to fully capture the target's complex distribution. 
To address this, we leverage \textit{ensembles of draft models}, where each constituent member acts as a lightweight approximation to a different mode of the target model's posterior predictive distribution, collectively capturing the multi-modal nature of it.
Under this view, disagreement among draft ensembles provides a natural proxy for epistemic uncertainty: when inputs are well-supported by the training distribution, the ensembles tend to agree, whereas for ambiguous or OOD queries, their predictions diverge substantially.

Nevertheless, draft models are imperfect approximations of the target and inevitably introduce bias.
Thus, effective uncertainty estimation must account for both the diversity among drafts and their mismatch with the target model. Motivated by this observation, we ask the following question:
\begin{quote}
\emph{Can we leverage small draft models to estimate the epistemic uncertainty of a large target model both accurately and efficiently?}
\end{quote}

In this work, we answer this question affirmatively and propose a principled framework for epistemic uncertainty estimation using draft models.
Our approach treats a family of drafts as a proxy for the target model’s posterior predictive distribution and decomposes uncertainty estimation error into variance and bias components. 
We derive a bias–variance decomposition showing that epistemic uncertainty can be approximated by (i) the Jensen–Shannon divergence among draft predictions as a variance proxy, and (ii) the KL-divergence between the draft mixture and the target model as a bias proxy.

To further enhance posterior diversity among drafts, we introduce \emph{Online Stochastic Distillation (OSD)} combined with a \emph{Data-Diverse Drafts (DDD)} strategy, where drafts are trained on disjoint data partitions.
Unlike on-policy distillation methods such as GKD~\cite{agarwal2024onpolicy} or MiniLLM~\cite{gu2024minillm}, which tend to over-align drafts with the target and collapse epistemic signals, our approach preserves complementary inductive biases across drafts, leading to more reliable uncertainty estimates.

Empirically, we evaluate our method on GSM8K~\cite{cobbe2021training} and demonstrate a 37.67\% reduction in uncertainty estimation error (RMSE) compared to strong baselines. 
Moreover, our approach achieves hallucination detection performance competitive with heavy perturbation-based methods such as TokUR~\cite{zhang2025tokur}, while incurring negligible additional inference cost.
In summary, our contributions are:
\begin{itemize}     
    \item We propose a novel perspective that leverages small draft models as a proxy for the target model’s posterior predictive distribution.
    \item We derive a bias–variance decomposition for epistemic uncertainty estimation, grounding draft disagreement in a principled theoretical framework.
    \item We introduce Online Stochastic Distillation and Data-Diverse Drafts to preserve posterior diversity among drafts.
    \item We empirically demonstrate that our method enables accurate and efficient epistemic uncertainty estimation for large language models.
\end{itemize}
\section{Related Works}
\label{sec:related}

\paragraph{Uncertainty Estimation in LLMs.}
Uncertainty Estimation for LLMs has been actively studied with the goal of quantifying the reliability of model outputs by estimating the likelihood of prediction errors in advance \cite{gawlikowski2023survey}. By explicitly measuring the uncertainty associated with LLM predictions, UE aims to mitigate hallucinations and to enable safer and more reliable decision-making in downstream applications \cite{xia2025survey}.


To this end, prior work has proposed a variety of methods for assessing output confidence, including approaches that exploit token- and sequence-level likelihood-based signals—such as predicted probabilities, entropy, and calibration metrics \cite{jiang2021calibrationQA, ahdritz2024knowable}—as well as approaches that estimate uncertainty from the variability of predictive distributions obtained via sampling and inference-time ensembling under prompt and input perturbations \cite{lyu2025sampleConsistency, manakul2023selfcheckgpt, tonolini2024bayesianPromptEnsembles}.


In particular, Bayesian Neural Network based approaches \cite{neal2012bayesian} distinguish uncertainty into \textit{aleatoric uncertainty}, arising from the intrinsic ambiguity of the data, and \textit{epistemic uncertainty}, stemming from limitations in the model’s knowledge and reasoning capacity, and attempt to quantify both through Bayesian posterior inference \cite{kendall2017uncertainties, hullermeier2021aleatoric}. To improve computational efficiency, recent work has proposed approximating Bayesian uncertainty by modeling only a subset of parameters probabilistically, for example through low-rank adapters or lightweight Bayesian modules \cite{dialameh2025bayesianMoE, yang2024bayesian, wang2024blob}.


However, much of the existing work has focused on short-form predictions or query-level uncertainty, and token-level uncertainty over long autoregressive generation trajectories remains underexplored \cite{gal2016dropout, malinin2021autoregressiveUE, shi2025trainingFreeBayesianization}. TokUR \cite{zhang2025tokur} takes a step toward addressing this gap by approximating token-level epistemic uncertainty via LoRA-based parameter perturbations, and shows that such uncertainty signals can be used to identify erroneous decoding trajectories during generation \cite{hu2022lora, gao2024spuq}.


Despite their effectiveness, perturbation-based epistemic UE methods require multiple stochastic forward passes of the target LLM to obtain stable uncertainty estimates, leading to a substantial increase in inference cost. This motivates the need for UE methodologies that can approximate epistemic uncertainty under significantly lower computational budgets.

\paragraph{Efficient Inference via Speculative Decoding.}
Speculative decoding \cite{stern2018blockwise, xia2023speculative} accelerates inference for LLMs by using a lightweight \textit{draft model} to propose multiple future tokens in advance, which are then verified by a larger \textit{target model} in a single forward pass. By reducing the number of expensive autoregressive decoding steps executed by the target model, this framework substantially improves decoding throughput without degrading output quality.


The effectiveness of speculative decoding is largely governed by how well the conditional distributions produced by the \textit{draft model} align with those of the \textit{target model}. Building on this principle, subsequent work has extended the framework with a range of drafting and verification strategies that trade off proposal quality against verification cost \cite{xia2024unlockingEfficiencySpecDec, hu2025speculativeSurvey}. 

Recently, several studies have introduced uncertainty-aware control policies for speculative decoding—such as dynamically adapting the verification frequency or draft length based on uncertainty signals—but these methods are primarily aimed at improving inference efficiency or refining verification strategies, rather than at estimating the epistemic uncertainty of the \textit{target model} itself \cite{qin2025dynamicWidthSpeculativeBeam, huang2025specdecpp, agrawal2024adaedl}.


In this work, we reinterpret speculative decoding from a distributional perspective and show that a properly trained family of draft models can be viewed as a finite surrogate for approximating the posterior of the target model. Building on this interpretation, we reframe speculative decoding into a structural framework for estimating the epistemic uncertainty of the target model in an effective and computationally efficient manner.


\paragraph{Knowledge Distillation.} 

Knowledge Distillation (KD) has been widely studied as a core technique for compressing LLMs and accelerating inference \cite{hinton2015distilling, yang2025surveyKDLLMs}. In a distributional sense, KD can be interpreted as a methodology for training a student model to approximate the teacher’s posterior-induced predictive distribution by matching its output distribution to that of the teacher.


However, standard distillation based on the forward Kullback–Leibler (KL) divergence tends to encourage the student to cover all modes of the teacher distribution. For capacity-limited student models, this mode-covering behavior can lead to overly diffuse or distorted predictive distributions \cite{anonymous2025onllmknowledge,agarwal2024onpolicy}.


To address this limitation, recent work has proposed on-policy distillation schemes, in which the student is trained on samples generated by its own policy and receives supervisory signals from the teacher. MiniLLM \cite{gu2024minillm} adopts reverse KL–based training to promote mode-seeking behavior in generative distillation, thereby encouraging the student to concentrate on the dominant modes of the teacher distribution. GKD \cite{agarwal2024onpolicy}  generalizes this paradigm by unifying divergence choices and sampling policies into a flexible framework, showing that the student distribution can be more faithfully aligned with the teacher’s posterior.
These on-policy distillation approaches provide a conceptual foundation for training draft models to better approximate the dominant probability mass structure of a target model, thereby offering theoretical support for draft-based uncertainty estimation.

However, while mode-seeking behaviors (encouraged by Reverse KL) are beneficial for generating high-quality text by avoiding unlikely tokens, they are detrimental for uncertainty estimation. This is because they artificially collapse the distributional variance necessary to represent the full posterior, thereby masking valid epistemic signals and leading to over-confident student models.


\section{Theoretical Analysis}
\label{sec:theory}
In this section, we establish a theoretical framework for estimating the Epistemic Uncertainty (EU) of a target Large Language Model (LLM) using a family of draft models. We formally define the uncertainty, derive a general upper bound utilizing the draft mixture, and propose a computationally efficient estimator based on a bias-variance decomposition.

\subsection{Preliminaries: \emph{Epistemic Uncertainty}}
Let $\theta$ denote the parameters of the target model, which are assumed to follow a posterior distribution $\pi_T(\theta)$ given the training data. The likelihood of generating an output sequence $y$ given parameters $\theta$ is $p_\theta(y)$. The \textit{predictive distribution} of the target model, often referred to as the Bayesian Model Average (BMA), is defined as the expectation over the posterior:
\begin{equation*}
    p_T(y) = \mathbb{E}_{\theta \sim \pi_T(\theta)}[p_\theta(y)]
\end{equation*}
Epistemic Uncertainty (EU) quantifies the reduction in uncertainty about the model parameters given the data. Formally, it is defined as the mutual information between the output $y$ and the parameters $\theta$, denoted as $I(y; \theta)$. This can be expressed as the difference between the total uncertainty (i.e., entropy of the predictive distribution) and the aleatoric uncertainty (i.e., expected entropy of individual model realizations):
\begin{equation*}
    \text{EU} = H(p_T) - \mathbb{E}_{\theta \sim \pi_T(\theta)}[H(p_\theta)]
\end{equation*}
Following \citet{Lakshminarayanan2017}, this mutual information is equivalent to the expected Kullback-Leibler (KL) divergence between individual parameter realizations and the predictive average.
\begin{lemma}
\label{lemma:eu_kl}
Epistemic uncertainty is equivalent to:
\begin{equation*}
    \text{EU} = \mathbb{E}_{\theta \sim \pi_T(\theta)}[KL(p_\theta || p_T)]
\end{equation*}
\end{lemma}
Computing this directly is computationally prohibitive for LLMs as it requires sampling multiple high-cost target models to approximate the expectation $\mathbb{E}_\theta$ and computing the ensemble average $p_T$.

\subsection{Bounding EU via Draft Mixture}
To reduce computational costs, we introduce a family of $K$ lightweight \emph{draft models} $\{q_k\}_{k=1}^K$, typically used in speculative decoding. We define their mixture distribution as $q_\text{mix}(y) = \frac{1}{K}\sum_{k=1}^K q_k(y)$.

We first derive a general relationship linking the target's epistemic uncertainty to this draft mixture.
\begin{theorem}[General Upper Bound]
The expected divergence between the target model realizations and the draft mixture provides an upper bound on the epistemic uncertainty:
\begin{equation*}
    \mathbb{E}_{\theta}[KL(p_\theta || q_\text{mix})] = \text{EU} + KL(p_T || q_\text{mix})
\end{equation*}
\end{theorem}
\begin{proof}
    Expanding the KL divergence term:
    \begin{align*}
        &\mathbb{E}_{\theta}[KL(p_\theta || q_\text{mix})] \\
        &= \mathbb{E}_{\theta}\left[\sum_y p_\theta(y) \log \frac{p_\theta(y)}{q_\text{mix}(y)}\right] \\
        &= \mathbb{E}_{\theta}\left[\sum_y p_\theta(y) \log \left(\frac{p_\theta(y)}{p_T(y)} \cdot \frac{p_T(y)}{q_\text{mix}(y)}\right)\right] \\
        &= \mathbb{E}_{\theta}\left[KL(p_\theta || p_T)\right] + \sum_y \mathbb{E}_{\theta}[p_\theta(y)] \log \frac{p_T(y)}{q_\text{mix}(y)}
    \end{align*}
    The first term is the definition of EU (Lemma \ref{lemma:eu_kl}). Since $\mathbb{E}_{\theta}[p_\theta(y)] = p_T(y)$, the second term simplifies to $KL(p_T || q_\text{mix})$.
\end{proof}
Since $KL(p_T || q_\text{mix}) \ge 0$, this confirms that calculating divergences against the draft mixture overestimates the true uncertainty unless the draft mixture perfectly matches the target average.

\subsection{Efficient Estimation via Bias-Variance Decomposition}
To avoid sampling $\theta$ entirely, we leverage the draft models as a proxy for the target posterior. We formally state the \emph{Proxy Posterior Assumption}:
\begin{assumption}

\label{ass:assumption1}
The expectation over the target posterior can be approximated by the expectation over the discrete set of draft models:
\begin{equation*}
    \mathbb{E}_{\theta \sim \pi_T(\theta)}[\cdot] \approx \mathbb{E}_{k \sim U(1, K)}[\cdot]
\end{equation*}
\end{assumption}

Under this assumption, we propose the proxy estimator $\hat{\text{EU}} = \mathbb{E}_k[KL(q_k || p_T)]$. We now show that this estimator can be decomposed into two computationally distinct terms.

\begin{theorem}[Bias-Variance Decomposition]
The proxy estimator decomposes exactly into the Jensen-Shannon Divergence (JSD) among the draft models (Variance) and the divergence of the draft mixture from the target (Bias):
\begin{equation}
\label{eq:bias_variance}
\resizebox{0.9\linewidth}{!}{$
    \mathbb{E}_k[KL(q_k || p_T)] = \underbrace{\text{JSD}(q_1, \dots, q_K)}_{\text{Variance Proxy}} + \underbrace{KL(q_\text{mix} || p_T)}_{\text{Bias Proxy}}
$}
\end{equation}
where $\text{JSD}(q_k) = \mathbb{E}_k[KL(q_k || q_\text{mix})]$.
\end{theorem}

\begin{proof}
    Substituting $q_\text{mix}$ into the log term of the estimator:
    \begin{align*}
        &\mathbb{E}_k\left[KL(q_k || p_T)\right] \\
        &= \mathbb{E}_k\left[\sum_y q_k \log \frac{q_k}{p_T}\right] \\
        &= \mathbb{E}_k\left[\sum_y q_k \log \left(\frac{q_k}{q_\text{mix}} \cdot \frac{q_\text{mix}}{p_T}\right)\right] \\
        &= \mathbb{E}_k[KL(q_k || q_\text{mix})] + \sum_y \mathbb{E}_k[q_k] \log \frac{q_\text{mix}}{p_T}
    \end{align*}
    The first term is the definition of JSD. Since $\mathbb{E}_k[q_k] = q_\text{mix}$, the second term becomes $KL(q_\text{mix} || p_T)$.
\end{proof}

\paragraph{Operational Significance.}
The decomposition in Eq. \ref{eq:bias_variance} transforms the computationally intractable EU estimation into two manageable terms:
\begin{enumerate}
    \item \textbf{Variance Proxy (JSD).} This measures the disagreement \textit{among} draft ensembles. While this requires forward passes of the constituent draft models, these computations are inherently part of the draft-assisted inference process and do not incur additional overhead relative to the inference strategy.
    \item \textbf{Bias Proxy (KL).} This measures the deviation of the draft mixture from the target posterior average ($p_T$). Crucially, by substituting the intractable $p_T$ with our distilled proxy $p_\text{mix}$, we can estimate this term with just \emph{a single forward pass of the proxy model}. This effectively replaces the prohibitive requirement of executing the full target model ensemble ($N$ passes) with a single inference step.
\end{enumerate}

\subsection{Learning the Proxy Mean via Online Stochastic Distillation}
\label{sec:osd_theory}

To practically compute the Bias Proxy term derived in Eq. \ref{eq:bias_variance}, we require a tractable approximation of the target's predictive distribution $p_T$. Since $p_T$ is an integration over the intractable posterior $\pi_T(\theta)$, direct computation is impossible. To address this, we introduce \textit{Online Stochastic Distillation} (OSD).

\paragraph{Definition and Objective.}
Standard Knowledge Distillation (KD) aligns a student with a fixed teacher. In contrast, OSD trains a single proxy student $p_\text{mix}$ (parameterized by $\phi$) to approximate the \textit{expected} output of a stochastic teacher distribution. We minimize the expected forward KL divergence over the target posterior $\pi_T(\theta)$:
\begin{equation*}
    \mathcal{L}_{\text{OSD}}(\phi) = \mathbb{E}_{x \sim \mathcal{D}} \mathbb{E}_{\theta \sim \pi_T(\theta)} \left[ KL(p_\theta(\cdot|x) || p_\text{mix}(\cdot|x; \phi)) \right]
\end{equation*}
In practice, at each training step, we sample a perturbed realization $\theta \sim \pi_T(\theta)$ (via low-rank noise injection) and update $p_\text{mix}$ to match this stochastic target.

\paragraph{Theoretical Justification.}
This objective is theoretically grounded in information geometry. It is a known property of the forward Kullback-Leibler divergence that the distribution $q^*$ minimizing the expected divergence from a set of distributions $\{p_\theta\}$ is their arithmetic mean (mixture):
\begin{equation*}
    \underset{q}{\arg\min} \ \mathbb{E}_{\theta}[KL(p_\theta || q)] = \mathbb{E}_{\theta}[p_\theta] = p_T
\end{equation*}
Thus, by training on stochastic teacher samples via Forward KL, $p_\text{mix}$ naturally converges to the Bayesian Model Average $p_T$, effectively marginalizing out the epistemic uncertainty of the teacher.
\section{Experimental Setup}
\label{sec:setup}
We designed a comprehensive suite of experiments to validate our theoretical framework and determine the optimal configuration for efficient EU estimation and its application to downstream tasks.

\subsection{Models and Posterior Simulation}
We utilized the Llama3 family of models \cite{touvron2023llama} to cover various scale ratios. Specifically, we employed Llama-3.1-8B-Instruct as the target model (i.e., Teacher) and Llama-3.2-3B-Instruct as the draft model (i.e., Student). 

To create a ground-truth \textit{Target Family} representing the posterior $\pi_T(\theta)$, we injected Low-rank Gaussian Noise into the parameters of the 8B target model. This process generates a set of perturbed models that act as samples from the posterior. We generated target families of size 3 to evaluate estimation performance under different posterior densities. 

\subsection{Datasets and Data Generation}
Experiments were conducted on the {GSM8K} dataset. To train the draft models and the proxy $p_{mix}$, we generated a synthetic training set where the Target model produced 4 diverse responses for each query in GSM8K. This multi-output data captures the aleatoric and epistemic uncertainty of the teacher, providing the necessary signals for distillation.

\subsection{Distillation Methods}
\label{sec:distillation_methods}
We employ \textit{Online Stochastic Distillation} (OSD) as our core training framework, which minimizes the expected Forward-KL divergence between the student and a stochastic teacher distribution. 

\paragraph{Training the Proxy $p_\text{mix}$.} To obtain the proxy model for the Bias term, we apply OSD on the full dataset. By training a single student against stochastic teacher samples, $p_{mix}$ converges to the Bayesian Model Average (mean) of the target family, efficiently approximating $p_T$ in Eq.~\ref{eq:bias_variance}.

\paragraph{Training the Draft Family.}
To evaluate how different sources of stochasticity affect posterior approximation, we categorize the draft training into three distinct variants:
\begin{itemize}
    \item \textbf{Data-Diverse Drafts (DDD):} Each draft model is trained on a \textit{disjoint partition} of the target-generated dataset, inducing diversity through variations in the training data distribution.
    \item \textbf{Initialization-Diverse Drafts (IDD):} Each draft model is initialized with diverse parameter noise and trained on the \textit{entire} dataset, isolating the effect of stochastic starting points on posterior coverage.
    \item \textbf{Fully-Diverse Drafts (FDD):} FDD combines both strategies by applying diverse parameter initialization and training each model on disjoint data partitions.
\end{itemize}

\paragraph{Baselines.} We compare OSD against \textit{GKD \cite{agarwal2024onpolicy}} and \textit{MiniLLM \cite{gu2024minillm}}. While these methods are effective for text quality, their \textit{mode-seeking} nature (Reverse-KL) often collapses the distribution, potentially hindering the capture of full epistemic uncertainty compared to the \textit{mode-covering} nature of OSD.

\subsection{Ensemble Configuration}
\label{sec:architectures}
To investigate the interplay between distillation strategies and ensemble construction, we focus on a structured 2$\times$3 configuration designed to capture the diversity of the target posterior. In this setup, we train six independent draft models and ensemble them in pairs to form three \textit{imitated groups}. Each group is specifically optimized to collectively mimic the predictive behavior of a corresponding member within the target family. This hierarchical approach allows the ensemble to effectively preserve the epistemic uncertainty of the target model by leveraging both the diversity induced during independent training and the collaborative refinement provided by group ensembling.

\subsection{Evaluation Protocols and Metrics}
\paragraph{Uncertainty Estimation.} We measure the fidelity of estimated EU against the ground-truth Target EU using \textit{RMSE} (Lower is better) and \textit{Spearman's $\rho$} (Higher is better). We also report the standard deviation across 5 repeated trials to ensure robustness.

\paragraph{Hallucination Detection.} We utilize token-level uncertainty to identify incorrect generations. Since raw scores are not calibrated probabilities, we train a \textit{logistic regression classifier} on a hold-out set to map uncertainty to correctness probabilities. Performance is measured via \textit{AUROC} (discrimination), \textit{ECE}, and \textit{Brier Score} (calibration). We compare our results against \textit{TokUR \cite{zhang2025tokur}}, a heavy baseline requiring multiple target-level perturbations.
\section{Experimental Results and Analysis}
\label{sec:results}

In this section, we present a detailed analysis of our experimental findings, demonstrating the effectiveness of our framework in approximating epistemic uncertainty and its practical utility in downstream tasks. We first establish the superiority of our Data-Diverse Drafts (DDD) strategy in uncertainty estimation and then evaluate its performance in hallucination detection.

\subsection{Superiority in Uncertainty Estimation}
\label{sec:exp_uncertainty}

We first validate the ability of our draft models to approximate the target model's epistemic uncertainty. This involves assessing the fidelity of the estimated values across various model capacities and distillation strategies.

\begin{table}[t]
\caption{\textbf{Uncertainty Estimation Performance.} Comparison of RMSE ($\downarrow$) and Spearman Correlation ($\uparrow$) across different methods.}
\label{tab:ue_8b_3b}
\centering
\begin{tabular}{l|cc}
\toprule
\textbf{Method} & \textbf{RMSE} $\downarrow$ & \textbf{Spearman} $\uparrow$ \\
\midrule\midrule
Baseline(No Train) & 0.3266 & 0.8781 \\
GKD                & 0.3107 & 0.8981 \\
MiniLLM            & 0.3029 & 0.8994 \\
IDD                & 0.3166 & 0.9159 \\
FDD                & 0.3140 & \textbf{0.9175} \\
\textbf{DDD (Ours)} & \textbf{0.2036} & 0.9165 \\
\bottomrule
\end{tabular}
\end{table}

\begin{table}[t]
\caption{\textbf{Hallucination Detection: Performance vs. Efficiency Trade-off.} Our lightweight draft-based method matches the heavy TokUR baseline while requiring significantly fewer FLOPs. \colorbox{blue!8}{Shaded rows} indicate our method. $\dagger$: Uses 8B target model. $\ddagger$: Uses lightweight drafts.}
\label{tab:hallucination_v2}
\centering
\resizebox{\columnwidth}{!}{
\begin{tabular}{l|c|ccc}
\toprule
\textbf{Method} & \textbf{Cost} & \textbf{AUROC}$\uparrow$ & \textbf{ECE}$\downarrow$ & \textbf{Brier}$\downarrow$ \\
\midrule\midrule
TokUR$^\dagger$ & 1.00$\times$ & 0.7823 & 0.0652 & 0.1495 \\
Draft Untrained$^\ddagger$ & 0.75$\times$ & \textbf{0.7853} & 0.0760 & 0.1494 \\
\cellcolor{blue!8}\textbf{DDD (Ours)}$^\ddagger$ & \cellcolor{blue!8}\textbf{0.75$\times$} & \cellcolor{blue!8}0.7839 & \cellcolor{blue!8}\textbf{0.0576} & \cellcolor{blue!8}\textbf{0.1480} \\
\bottomrule
\end{tabular}
}
\end{table}

As summarized in Table \ref{tab:ue_8b_3b}, our \textit{DDD strategy} consistently achieves the lowest RMSE and highest Spearman correlation across all model pairs. DDD reduces RMSE to {0.2036}, significantly outperforming advanced distillation methods such as GKD (0.3107) and MiniLLM (0.3029). This suggests that our strategy effectively forces the draft family to cover the diverse modes of the target posterior, which is essential for accurate uncertainty quantification.

\begin{table*}[t]
\caption{\textbf{Uncertainty Estimation Performance (Smaller Draft Models).} Comparison of RMSE ($\downarrow$) and Spearman Correlation ($\uparrow$) for 1B draft models. Our \textbf{DDD} strategy consistently outperforms baselines in capturing the target's uncertainty across different target-draft scales.}
\label{tab:ue_1b}
\centering
\resizebox{0.65\textwidth}{!}{ 
\begin{tabular}{l|cc|cc}
\toprule
\multicolumn{1}{c}{}& \multicolumn{2}{c}{\textbf{8B (Target) $\rightarrow$ 1B (Draft)}} & \multicolumn{2}{c}{\textbf{3B (Target) $\rightarrow$ 1B (Draft)}} \\
\cmidrule(lr){2-3}\cmidrule(lr){4-5}
\multicolumn{1}{c}{\textbf{Method}} & \textbf{RMSE} $\downarrow$ & \multicolumn{1}{c}{\textbf{Spearman} $\uparrow$} & \textbf{RMSE} $\downarrow$ & \textbf{Spearman} $\uparrow$ \\
\midrule\midrule
Baseline(No Train) & 0.5719 & 0.8189 & 0.4696 & 0.8327 \\
IDD                & 0.5927 & 0.8843 & 0.5875 & 0.8881 \\
FDD                & 0.5638 & 0.8899 & 0.5604 & 0.8901 \\
\textbf{DDD (Ours)} & \textbf{0.4437} & \textbf{0.8951} & \textbf{0.4023} & \textbf{0.9017} \\
\bottomrule
\end{tabular}
}
\end{table*}

\subsection{Application: Hallucination Detection}
\label{sec:exp_hallucination}

We assess the practical utility of our method by using the estimated token-level uncertainty to identify incorrect generations (hallucinations). We compare our lightweight \textbf{Draft Trained} approach against \textbf{TokUR} \cite{zhang2025tokur}, a heavy baseline that requires multiple forward passes of the target model.

Table \ref{tab:hallucination_v2} presents the detection results on the GSM8K dataset. Our method achieves an AUROC of 0.7839, effectively matching the TokUR baseline (0.7823), and even shows an improvement in ECE and Brier Score. While TokUR incurs significant computational latency by perturbing the 8B target, our method utilizes the 3B draft models. This allows for high-quality uncertainty quantification with negligible overhead.

\subsection{Ablation Studies}
\label{sec:ablation}
\paragraph{Scalability Across Model Sizes.}
We further evaluate the robustness of our approach by applying it to draft with significant low capacity. We use Llama-3.2-1B-Instruct as the draft, and test both scenarios where 8B and 3B model serves as the target. As shown in Table~\ref{tab:ue_1b}, despite the 1B draft model having substantially fewer parameters, the DDD strategy maintains a strong Spearman correlation of {0.8951} in the 8B$\rightarrow$1B scenario. While a slight performance dip occurs in the 8B$\rightarrow$1B scenario due to the increased capacity gap of the target and the draft, the consistent superiority of DDD over baselines highlights its robustness.

\paragraph{Role of the Distilled Proxy $p_\text{mix}$.}
We evaluate the impact of the distilled proxy $p_\text{mix}$ on stabilizing the bias term estimation. As shown in Table~\ref{tab:pmix_robustness}, $p_\text{mix}$ demonstrates superior efficiency and robustness compared to the raw target model ($p_\text{orig}$). 
Crucially, our method achieves a lower RMSE of $0.0283$ with only $K=3$ perturbations, significantly outperforming the raw target model (RMSE 0.0391) even when the latter utilizes a much larger ensemble of $K=10$. 
This result highlights the efficiency of our approach: $p_\text{mix}$ effectively marginalizes over the target's aleatoric noise, allowing for high-fidelity uncertainty estimation with minimal inference overhead. 
This stability is further visualized in Figure~\ref{fig:pmix_stability_boxplot_8b}; notably, when matching the perturbation budget ($K=10$), our proxy reduces the RMSE further to $\approx 0.014$, confirming its effectiveness as a reliable auditor.

\begin{table}[t]
\caption{\textbf{Robustness of $p_\text{mix}$ Proxy.} Comparison of RMSE ($\downarrow$) and Spearman correlation ($\uparrow$) between the raw target model and our distilled $p_\text{mix}$ proxy under varying perturbations.}
\label{tab:pmix_robustness}
\centering
\resizebox{\columnwidth}{!}{
\begin{tabular}{l|l|cc}
\toprule
\textbf{Perturbation} & \textbf{Metric} & \textbf{Target (raw)} & \textbf{$p_\text{mix}$ (Ours)} \\
\midrule\midrule
\multirow{2}{*}{\textbf{K=3}} 
& RMSE ($\downarrow$) & 0.0436 & \textbf{0.0283} \\
& Spearman ($\uparrow$) & \textbf{0.9945} & 0.9829 \\
\midrule
\multirow{2}{*}{\textbf{K=10}}
& RMSE ($\downarrow$) & 0.0391 & \textbf{0.0137} \\
& Spearman ($\uparrow$) & 0.9957 & \textbf{0.9969} \\
\bottomrule
\end{tabular}
}
\end{table}

\begin{figure}[t]
  \centering
  \includegraphics[width=\columnwidth]{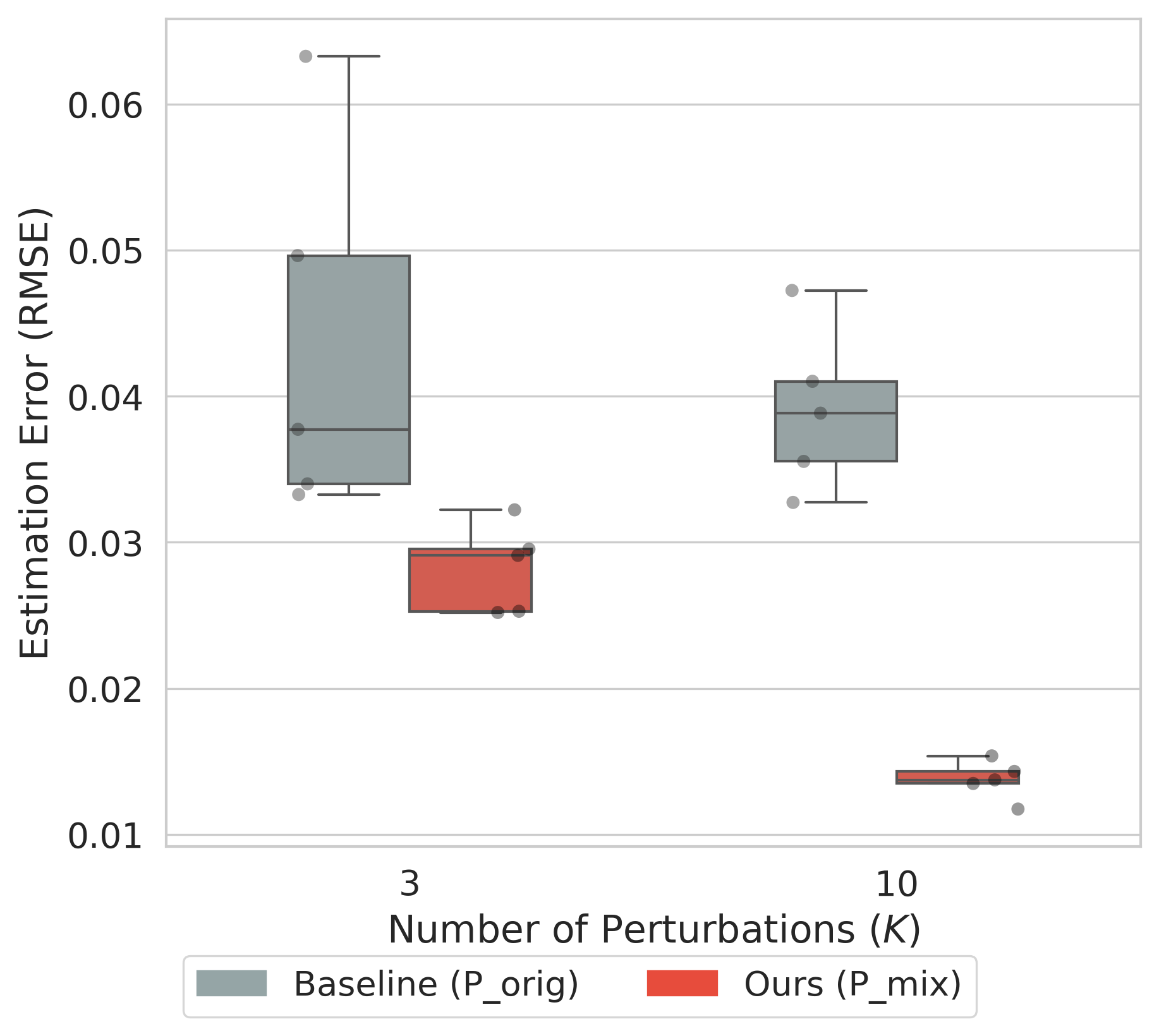}
  \caption{\textbf{Stability analysis comparing $p_{\text{orig}}$ (baseline) and $p_{\text{mix}}$ (Ours).} Box plots show RMSE distribution over 5 runs under $K=3$ and $K=10$ perturbations. Lower RMSE indicates more stable estimation.}
  \label{fig:pmix_stability_boxplot_8b}
\end{figure}

\paragraph{Impact of Data Partitioning and Architecture.}
We analyze the sensitivity of the DDD strategy to data splits. Table~\ref{tab:data_splits} confirms that data partitioning is critical for posterior coverage; the $2\times3$ and $3\times3$ configurations consistently outperform single-chunk training. 
It is worth noting that while the $3\times3$ configuration yields the marginally lowest error (0.2024), the $2\times3$ architecture achieves a highly competitive error (0.2036) with reduced training complexity. 
Therefore, we select $2\times3$ as the optimal trade-off between performance and resource efficiency. For detailed explanation of data partitioning strategies, see Appendix~\ref{app:detail_ensemble}.

\begin{table}[t]
\caption{\textbf{Impact of Data Partitioning.} Evaluation of different split configurations on 8B$\rightarrow$3B RMSE ($\downarrow$). Data-diverse partitioning consistently reduces estimation error.}
\label{tab:data_splits}
\centering
\small 
\setlength{\tabcolsep}{10pt} 
\begin{tabular}{lc}
\toprule
\textbf{Split Config} & \textbf{RMSE} ($\downarrow$) \\
\midrule\midrule
$1\times3$ & 0.2050 \\
\textbf{2$\times$3 (Ours)} & 0.2036 \\
$3\times3$ & \textbf{0.2024} \\
3 only & 0.2122 \\
6 only & 0.2145 \\
9 only & 0.2129 \\
\bottomrule
\end{tabular}
\end{table}

\subsection{Computational Analysis}
\label{sec:compute}

A key advantage of our framework is its computational efficiency. We compare the inference cost of our method against the TokUR baseline, which requires multiple stochastic forward passes of the heavy target model.

\paragraph{FLOPs Analysis.}
We estimate the total inference cost in terms of parameter-equivalent FLOPs. The total cost for our method comprises the drafts' variance estimation plus a single forward pass of the proxy target (required for the bias term and generation). Table~\ref{tab:compute} summarizes the relative overhead. The TokUR baseline requires $K=3$ forward passes of the 8B target model, resulting in a total cost of $3 \times 8\text{B} = 24\text{B}$ per sequence. In comparison, our 8B$\rightarrow$1B configuration (DDD) utilizes 6 lightweight drafts plus one target pass, totaling $(6 \times 1\text{B}) + (1 \times 8\text{B}) = 14\text{B}$. This achieves a 42\% reduction in FLOPs ($0.58\times$) compared to TokUR, while maintaining competitive detection performance. While the 8B$\rightarrow$3B configuration incurs a cost comparable to the baseline ($1.08\times$), it offers higher estimation fidelity for tasks prioritizing precision over raw speed.


\begin{table}[t]
\caption{\textbf{Computational Cost Comparison.} Relative FLOPs for uncertainty estimation. The Total Cost includes 6 draft passes for variance and 1 target pass for bias/generation. Our 1B configuration achieves competitive performance with significantly reduced overhead ($0.58\times$).}
\label{tab:compute}
\centering
\resizebox{\columnwidth}{!}{
\begin{tabular}{l|ccc|c}
\toprule
\textbf{Method} & \textbf{Config} & \textbf{\# Passes} & \textbf{Rel. FLOPs} & \textbf{AUROC} \\
\midrule\midrule
TokUR & 8B (Target) & 3 & 1.00$\times$ & 0.7823 \\
\midrule
\rowcolor{gray!10} DDD (Ours) & 3B $\times$ 6 & 7 & 1.08$\times$ & 0.7839 \\
\rowcolor{gray!10} DDD (Ours) & 1B $\times$ 6 & 7 & \textbf{0.58$\times$} & 0.7781 \\
\bottomrule
\end{tabular}
}
\end{table}
\section{Discussion}
\label{sec:discussion}

\paragraph{Validity of the Bias-Variance Decomposition.}
Empirical results strongly validate the theoretical framework in Section \ref{sec:theory}. High correlation (CCC $> 0.9$) confirms that epistemic uncertainty can be reliably factorized into the internal disagreement among draft models (Variance) and their collective shift from the target (Bias), proving that our proxy-based approach captures the teacher's uncertainty structure.

\paragraph{The Role of Data Diversity.}
We find that random initialization alone fails to induce draft diversity. Without distinct data views, draft models collapse toward the target mean, causing the Variance Proxy ($JSD \approx 0$) to vanish. DDD effectively prevents this by forcing students to learn diverse, valid hypotheses, successfully mapping the teacher's posterior uncertainty onto the student ensemble.

\paragraph{Lightweight Models as Effective Auditors.}
Our method’s competitive performance in hallucination detection suggests a strong distributional alignment between target and draft families. Notably, 1B models can effectively "audit" 8B teachers (Table \ref{tab:hallucination_v2}); while smaller models lack full generative capacity, they reliably preserve relative uncertainty rankings, serving as efficient proxies for identifying target "confusion."

\paragraph{Efficiency and Practical Integration.}
By replacing $N$ target passes with $N$ 3B draft passes and one proxy pass, we significantly reduce costs. Since draft models are typically memory-resident for speculative decoding, this approach enables near-"free" uncertainty estimation with negligible overhead compared to standard Deep Ensembles.
\section{Conclusion}
\label{sec:conclusion}
In this paper, we presented a framework to efficiently estimate the Epistemic Uncertainty of LLMs by leveraging lightweight draft models inherent to efficient inference paradigms. Our experimental results confirm that the \textbf{Data-Diverse Drafts (DDD)} strategy significantly outperforms existing distillation methods in capturing posterior diversity, while the distilled proxy ($p_\text{mix}$) ensures robust estimation by reducing variance by up to 90\%. Furthermore, we demonstrate that this draft-based uncertainty is highly effective for \textbf{Hallucination Detection}, achieving performance comparable to computationally expensive baselines like TokUR with zero marginal inference cost. By bridging the gap between efficient inference and reliable uncertainty quantification, our framework provides a practical and scalable solution for ensuring the safety and reliability of LLMs in real-world deployment.


\section*{Impact Statement}
This paper presents work whose goal is to advance the field of Machine Learning by enabling reliable and efficient uncertainty quantification for Large Language Models. As LLMs are increasingly integrated into safety-critical domains, the ability to detect and mitigate hallucinations is essential for responsible AI deployment. Our framework democratizes access to robust uncertainty signals by leveraging existing inference acceleration infrastructures, thereby eliminating the prohibitive computational costs typically associated with high-quality uncertainty estimation. By providing a scalable solution for real-time hallucination detection with zero marginal inference cost, this work contributes to the development of more trustworthy AI systems and facilitates their safe adoption across diverse industrial and societal applications.

\bibliography{references}

@article{gawlikowski2023survey,
  title={{A Survey of Uncertainty in Deep Neural Networks}},
  author={Gawlikowski, Jakob and Tassi, Cedrique Rovile Njieutcheu and Ali, Mohsin and Lee, Jongseok and Humt, Matthias and Feng, Jianxiang and Kruspe, Anna and Triebel, Rudolph and Jung, Peter and Roscher, Ribana and Shahzad, Muhammad and Yang, Wen and Bamler, Richard and Zhu, Xiao Xiang},
  journal={Artificial Intelligence Review},
  year={2023}
}

@inproceedings{xia2025survey,
  title={{A Survey of Uncertainty Estimation Methods on Large Language Models}},
  author={Xia, Zhiqiu and Xu, Jinxuan and Zhang, Yuqian and Liu, Hang},
  booktitle={ACL Findings},
  year={2025}
}

@inproceedings{ahdritz2024knowable,
  title={{Distinguishing the Knowable from the Unknowable with Language Models}},
  author={Ahdritz, Gustaf and Qin, Tian and Vyas, Nikhil and Barak, Boaz and Edelman, Benjamin L.},
  booktitle={ICML},
  year={2024}
}

@article{jiang2021calibrationQA,
  title={{How Can We Know When Language Models Know? On the Calibration of Language Models for Question Answering}},
  author={Jiang, Zhengbao and Araki, Jun and Ding, Haibo and Neubig, Graham},
  journal={Transactions of the Association for Computational Linguistics},
  year={2021}
}

@inproceedings{lyu2025sampleConsistency,
  title={{Calibrating Large Language Models with Sample Consistency}},
  author={Lyu, Qing and Shridhar, Kumar and Malaviya, Chaitanya and Zhang, Li and Elazar, Yanai and Tandon, Niket and Apidianaki, Marianna and Sachan, Mrinmaya and Callison-Burch, Chris},
  booktitle={AAAI},
  year={2025}
}

@inproceedings{manakul2023selfcheckgpt,
  title={{SelfCheckGPT: Zero-Resource Black-Box Hallucination Detection for Generative Large Language Models}},
  author={Manakul, Potsawee and Liusie, Adian and Gales, Mark},
  booktitle={EMNLP},
  year={2023}
}

@inproceedings{tonolini2024bayesianPromptEnsembles,
  title={{Bayesian Prompt Ensembles: Model Uncertainty Estimation for Black-Box Large Language Models}},
  author={Tonolini, Francesco and Aletras, Nikolaos and Massiah, Jordan and Kazai, Gabriella},
  booktitle={ACL Findings},
  year={2024}
}

@book{neal2012bayesian,
  title={{Bayesian Learning for Neural Networks}},
  author={Neal, Radford M.},
  year={1996},
  publisher={Springer}
}

@inproceedings{kendall2017uncertainties,
  title={{What Uncertainties Do We Need in Bayesian Deep Learning for Computer Vision?}},
  author={Kendall, Alex and Gal, Yarin},
  booktitle={NeurIPS},
  year={2017}
}

@article{hullermeier2021aleatoric,
  title={{Aleatoric and Epistemic Uncertainty in Machine Learning: An Introduction to Concepts and Methods}},
  author={H{\"u}llermeier, Eyke and Waegeman, Willem},
  journal={Machine Learning},
  year={2021}
}

@article{dialameh2025bayesianMoE,
  title={{Bayesian Mixture of Experts for Large Language Models}},
  author={Dialameh, Maryam and Rajabzadeh, Hossein and Zhang, Weiwei and Ahmed, Walid and Kwon, Hyock Ju},
  year={2025},
  journal={arXiv:2511.08968}
}

@inproceedings{yang2024bayesian,
  title={{Bayesian Low-Rank Adaptation for Large Language Models}},
  author={Yang, Adam X. and Robeyns, Maxime and Wang, Xi and Aitchison, Laurence},
  booktitle={ICLR},
  year={2024}
}

@inproceedings{wang2024blob,
  title={{BLoB: Bayesian Low-Rank Adaptation by Backpropagation for Large Language Models}},
  author={Wang, Yibin and Shi, Haizhou and Han, Ligong and Metaxas, Dimitris and Wang, Hao},
  booktitle={NeurIPS},
  year={2024}
}

@article{zhang2025tokur,
  title={{TokUR: Token-Level Uncertainty Estimation for Large Language Model Reasoning}},
  author={Zhang, Tunyu and Shi, Haizhou and Wang, Yibin and Wang, Hengyi and He, Xiaoxiao and Li, Zhuowei and Chen, Haoxian and Han, Ligong and Xu, Kai and Zhang, Huan and Metaxas, Dimitris and Wang, Hao},
  year={2025},
  journal={arXiv:2505.11737}
}

@inproceedings{hu2022lora,
  title={{LoRA: Low-Rank Adaptation of Large Language Models}},
  author={Hu, Edward J. and Shen, Yelong and Wallis, Phillip and Allen-Zhu, Zeyuan and Li, Yuanzhi and Wang, Shean and Wang, Lu and Chen, Weizhu},
  booktitle={ICLR},
  year={2022}
}

@inproceedings{gao2024spuq,
  title={{SPUQ: Perturbation-Based Uncertainty Quantification for Large Language Models}},
  author={Gao, Xiang and Zhang, Jiaxin and Mouatadid, Lalla and Das, Kamalika},
  booktitle={EACL},
  year={2024}
}

@inproceedings{gal2016dropout,
  title={{Dropout as a Bayesian Approximation: Representing Model Uncertainty in Deep Learning}},
  author={Gal, Yarin and Ghahramani, Zoubin},
  booktitle={ICML},
  year={2016}
}

@inproceedings{malinin2021autoregressiveUE,
  title={{Uncertainty Estimation in Autoregressive Structured Prediction}},
  author={Malinin, Andrey and Gales, Mark},
  booktitle={ICLR},
  year={2021}
}

@inproceedings{shi2025trainingFreeBayesianization,
  title={{Training-Free Bayesianization for Low-Rank Adapters of Large Language Models}},
  author={Shi, Haizhou and Wang, Yibin and Han, Ligong and Zhang, Huan and Wang, Hao},
  booktitle={NeurIPS},
  year={2025}
}

@inproceedings{stern2018blockwise,
  title={{Blockwise Parallel Decoding for Deep Autoregressive Models}},
  author={Stern, Mitchell and Shazeer, Noam and Uszkoreit, Jakob},
  booktitle={NeurIPS},
  year={2018}
}

@inproceedings{qin2025dynamicWidthSpeculativeBeam,
  title={{Dynamic-Width Speculative Beam Decoding for LLM Inference}},
  author={Qin, Zongyue and He, Zifan and Prakriya, Neha and Cong, Jason and Sun, Yizhou},
  booktitle={AAAI},
  year={2025}
}

@inproceedings{huang2025specdecpp,
  title={{SpecDec++: Boosting Speculative Decoding via Adaptive Candidate Lengths}},
  author={Huang, Kaixuan and Guo, Xudong and Wang, Mengdi},
  booktitle={COLM},
  year={2025}
}

@inproceedings{agrawal2024adaedl,
  title={{AdaEDL: Early Draft Stopping for Speculative Decoding of Large Language Models via an Entropy-Based Lower Bound on Token Acceptance Probability}},
  author={Agrawal, Sudhanshu and Jeon, Wonseok and Lee, Mingu},
  booktitle={NeurIPS Workshops},
  year={2024}
}

@inproceedings{xia2024unlockingEfficiencySpecDec,
  title={{Unlocking Efficiency in Large Language Model Inference: A Comprehensive Survey of Speculative Decoding}},
  author={Xia, Heming and Yang, Zhe and Dong, Qingxiu and Wang, Peiyi and Li, Yongqi and Ge, Tao and Liu, Tianyu and Li, Wenjie and Sui, Zhifang},
  booktitle={ACL Findings},
  year={2024}
}

@article{hu2025speculativeSurvey,
  title={{Speculative Decoding and Beyond: An In-Depth Survey of Techniques}},
  author={Hu, Yunhai and Liu, Zining and Dong, Zhenyuan and Peng, Tianfan and McDanel, Bradley and Zhang, Sai Qian},
  year={2025},
  journal={arXiv:2502.19732}
}

@article{hinton2015distilling,
  title={{Distilling the Knowledge in a Neural Network}},
  author={Hinton, Geoffrey and Vinyals, Oriol and Dean, Jeff},
  year={2015},
  journal={arXiv:1503.02531}
}

@inproceedings{anonymous2025onllmknowledge,
  title={{On LLM Knowledge Distillation: A Comparison between Forward KL and Reverse KL}},
  author={Cao, Yihan and Kang, Yanbin},
  booktitle={ICLR Blogposts},
  year={2025}
}

@article{yang2025surveyKDLLMs,
  title={{Survey on Knowledge Distillation for Large Language Models: Methods, Evaluation, and Application}},
  author={Yang, Chuanpeng and Zhu, Yao and Lu, Wang and Wang, Yidong and Chen, Qian and Gao, Chenlong and Yan, Bingjie and Chen, Yiqiang},
  journal={ACM Transactions on Intelligent Systems and Technology},
  year={2025}
}

@inproceedings{agarwal2024onpolicy,
  title={{On-Policy Distillation of Language Models: Learning from Self-Generated Mistakes}},
  author={Agarwal, Rishabh and Vieillard, Nino and Zhou, Yongchao and Stanczyk, Piotr and Ramos Garea, Sabela and Geist, Matthieu and Bachem, Olivier},
  booktitle={ICLR},
  year={2024}
}

@inproceedings{gu2024minillm,
  title={{MiniLLM: Knowledge Distillation of Large Language Models}},
  author={Gu, Yuxian and Dong, Li and Wei, Furu and Huang, Minlie},
  booktitle={ICLR},
  year={2024}
}

@inproceedings{Lakshminarayanan2017,
  title={{Simple and Scalable Predictive Uncertainty Estimation Using Deep Ensembles}},
  author={Lakshminarayanan, Balaji and Pritzel, Alexander and Blundell, Charles},
  booktitle={NeurIPS},
  year={2017}
}

@inproceedings{Leviathan2023,
  title={{Fast Inference from Transformers via Speculative Decoding}},
  author={Leviathan, Yaniv and Kalman, Matan and Matias, Yossi},
  booktitle={ICML},
  year={2023}
}

@article{huang2025surveyHallucinationLLMs,
  title={{A Survey on Hallucination in Large Language Models: Principles, Taxonomy, Challenges, and Open Questions}},
  author={Huang, Lei and Yu, Weijiang and Ma, Weitao and Zhong, Weihong and Feng, Zhangyin and Wang, Haotian and Chen, Qianglong and Peng, Weihua and Feng, Xiaocheng and Qin, Bing and Liu, Ting},
  journal={ACM Transactions on Information Systems},
  year={2025}
}

@article{alansari2025hallucinationSurvey,
  title={{Large Language Models Hallucination: A Comprehensive Survey}},
  author={Alansari, Aisha and Luqman, Hamzah},
  year={2025},
  journal={arXiv:2510.06265}
}

@article{ji2023surveyHallucinationNLG,
  title={{Survey of Hallucination in Natural Language Generation}},
  author={Ji, Ziwei and Lee, Nayeon and Frieske, Rita and Yu, Tiezheng and Su, Dan and Xu, Yan and Ishii, Etsuko and Bang, Ye Jin and Madotto, Andrea and Fung, Pascale},
  journal={ACM Computing Surveys},
  year={2023}
}

@inproceedings{ovadia2019trustUncertainty,
  title={{Can You Trust Your Model's Uncertainty? Evaluating Predictive Uncertainty under Dataset Shift}},
  author={Ovadia, Yaniv and Fertig, Emily and Ren, Jie and Nado, Zachary and Sculley, D. and Nowozin, Sebastian and Dillon, Joshua V. and Lakshminarayanan, Balaji and Snoek, Jasper},
  booktitle={NeurIPS},
  year={2019}
}

@article{farquhar2024semanticEntropyHallucination,
  title={{Detecting Hallucinations in Large Language Models Using Semantic Entropy}},
  author={Farquhar, Sebastian and Kossen, Jannik and Kuhn, Lorenz and Gal, Yarin},
  journal={Nature},
  year={2024}
}

@article{gemini2023family,
  title={{Gemini: A Family of Highly Capable Multimodal Models}},
  author={{Gemini Team}},
  year={2023},
  journal={arXiv:2312.11805}
}

@article{bubeck2023sparks,
  title={{Sparks of Artificial General Intelligence: Early Experiments with GPT-4}},
  author={Bubeck, Sébastien and Chandrasekaran, Varun and Eldan, Ronen and Gehrke, Johannes and Horvitz, Eric and Kamar, Ece and Lee, Peter and Lee, Yin Tat and Li, Yuanzhi and Lundberg, Scott and Nori, Harsha and Palangi, Hamid and Ribeiro, Marco Tulio and Zhang, Yi},
  year={2023},
  journal={arXiv:2303.12712}
}

@article{cobbe2021training,
  title={{Training Verifiers to Solve Math Word Problems}},
  author={Cobbe, Karl and Kosaraju, Vineet and Bavarian, Mohammad and Chen, Mark and Jun, Heewoo and Kaiser, Lukasz and Plappert, Matthias and Tworek, Jerry and Hilton, Jacob and Nakano, Reiichiro and Hesse, Christopher and Schulman, John},
  year={2021},
  journal={arXiv:2110.14168}
}

@article{touvron2023llama,
  title={{LLaMA: Open and Efficient Foundation Language Models}},
  author={Touvron, Hugo and Lavril, Thibaut and Izacard, Gautier and Martinet, Xavier and Lachaux, Marie-Anne and Lacroix, Timothée and Rozière, Baptiste and Goyal, Naman and Hambro, Eric and Azhar, Faisal and Rodriguez, Aurelien and Joulin, Armand and Grave, Edouard and Lample, Guillaume},
  year={2023},
  journal={arXiv:2302.13971}
}

@inproceedings{xia2023speculative,
  title={{Speculative Decoding: Exploiting Speculative Execution for Accelerating Seq2seq Generation}},
  author={Xia, Heming and Ge, Tao and Wang, Peiyi and Chen, Si-Qing and Wei, Furu and Sui, Zhifang},
  booktitle={EMNLP Findings},
  year={2023}
}

@inproceedings{kuhn2023semantic,
  title={{Semantic Uncertainty: Linguistic Invariances for Uncertainty Estimation in Natural Language Generation}},
  author={Kuhn, Lorenz and Gal, Yarin and Farquhar, Sebastian},
  booktitle={ICLR},
  year={2023}
}
\bibliographystyle{icml2026}

\newpage
\appendix
\onecolumn
\section{Theoretical Proofs}
\label{app:proofs}
\subsection{Proof of Lemma \ref{lemma:eu_kl}}
\begin{lemma}
\label{lemma:eu_kl}
$\text{EU} = E_\theta[KL(p_\theta || p_T)]$
\end{lemma}
\begin{proof}
\begin{align*}
E_\theta[KL(p_\theta || p_T)] &= E_\theta\left[\sum_y p_\theta(y) \log \frac{p_\theta(y)}{p_T(y)}\right] \\
&= E_\theta[-H(p_\theta)] - \sum_y E_\theta[p_\theta(y)] \log p_T(y) \\
&= -E_\theta[H(p_\theta)] + H(p_T) = \text{EU}
\end{align*}
\end{proof}

\subsection{Proof of Theorem (Bias-Variance Decomposition)}
\begin{proof}
We decompose the estimator $E_k[KL(q_k || p_T)]$ by introducing $q_{mix}$:
\begin{align*}
E_k[KL(q_k || p_T)] &= E_k\left[\sum_y q_k(y) \log \left(\frac{q_k(y)}{q_{mix}(y)} \cdot \frac{q_{mix}(y)}{p_T(y)}\right)\right] \\
&= E_k\left[KL(q_k || q_{mix})\right] + \sum_y E_k[q_k(y)] \log \frac{q_{mix}(y)}{p_T(y)} \\
&= \text{JSD}(q_1, \dots, q_K) + KL(q_{mix} || p_T)
\end{align*}
\end{proof}

\clearpage

\section{Additional Experimental Details}
\label{app:exp_details}

\subsection{Training $p_\text{mix}$}
Both $p_\text{mix}$ models were trained under the same configuration described below. The training dataset was prepared as follows. For each noise-injected teacher model used for distillation, we generated five answers per question on the GSM8K dataset. This produced three datasets, which were then mixed in equal proportions (one third each) to form the final training set. During training, for each batch we sampled $N_\text{teacher}$ noise modules to inject into the teacher model, and computed the KL loss between the student model and each sampled teacher. The injected noise was drawn from $\mathcal{N}(0, 0.1)$ same as target perturb .

\begin{table}[h]
\caption{\textbf{Hyperparameter Details for $p_\text{mix}$ training.}}
\label{tab:hyperparameter_pmix}
\centering
\begin{tabular}{lc}
\toprule
\textbf{Hyperparameter} & \textbf{Value} \\
\midrule\midrule
Batch size & 8 \\
Gradient accumulation steps & 8 \\
Epochs & 12 \\
Weight decay & 0.01 \\
Warmup ratio & 0.1 \\
$Lora_\text{r}$ & 16 \\
$Lora_\alpha$ & 32 \\
$Lora_\text{dropout}$ & 0.1 \\
Lr & 3e-4 \\
$N_\text{teacher}$ & 2 \\
Lora modules & $q_\text{proj}$, $v_\text{proj}$ \\
\bottomrule
\end{tabular}
\end{table}

\subsection{Training IDD, DDD, FDD}
Draft models trained with the IDD, DDD, and FDD methods shared the following common training setup. For IDD and FDD, noise modules were sampled from $\mathcal{N}(0, 0.5)$ and injected into the draft models. Under IDD, the ensemble draft models were trained on the same dataset, whereas under DDD and FDD, each draft model was trained on a different dataset. Each dataset was constructed by sampling four answers per question from the GSM8K dataset.

\begin{table}[h]
\caption{\textbf{Hyperparameter Details for IDD, DDD, FDD training.}}
\label{tab:hyperparameter_IDD_DDD_FDD}
\centering
\begin{tabular}{lc}
\toprule
\textbf{Hyperparameter} & \textbf{Value} \\
\midrule\midrule
Batch size & 8 \\
Gradient accumulation steps & 2 \\
Epochs & 1 \\
Weight decay & 0 \\
Warmup steps & 150 \\
$Lora_\text{r}$ & 16 \\
$Lora_\alpha$ & 32 \\
$Lora_\text{dropout}$ & 0.1 \\
Lr & 7e-5 \\
Lora modules & $q_\text{proj}$, $k_\text{proj}$, $v_\text{proj}$, $o_\text{proj}$, $gate_\text{proj}$, $up_\text{proj}$, $down_\text{proj}$ \\
\bottomrule
\end{tabular}
\end{table}

\clearpage

\subsection{Training GKD}
For GKD, the draft model was trained with the Hugging Face GKD Trainer. Unlike IDD, DDD, and FDD—which are trained on datasets generated by the teacher—GKD follows an on-policy setup where the student generates its own training data. To avoid overly perturbing the student’s distribution during early training, we sampled the initial sampling noise from $\mathcal{N}(0, 0.1)$.

\begin{table}[h]
\caption{\textbf{Hyperparameter Details for GKD training.}}
\label{tab:hyperparameter_GKD}
\centering
\begin{tabular}{lc}
\toprule
\textbf{Hyperparameter} & \textbf{Value} \\
\midrule\midrule
Max new tokens & 400 \\
Temperature & 1.0 \\
$Top_\text{p}$ & 1.0 \\
Batch size & 8 \\
Gradient accumulation steps & 2 \\
Epochs & 1 \\
Weight decay & 0 \\
Warmup steps & 150 \\
Lr Scheduler & Cosine \\
$Lora_\text{r}$ & 16 \\
$Lora_\alpha$ & 32 \\
$Lora_\text{dropout}$ & 0.1 \\
Lr & 7e-5 \\
Lora modules & $q_\text{proj}$, $k_\text{proj}$, $v_\text{proj}$, $o_\text{proj}$, $gate_\text{proj}$, $up_\text{proj}$, $down_\text{proj}$ \\
$\lambda$ & 0.5 \\
$\beta$ & 0.5 \\
\bottomrule
\end{tabular}
\end{table}

\subsection{Training MiniLLM}
For MiniLLM, the draft model was trained using the Hugging Face MiniLLM Trainer. As with GKD, MiniLLM is trained in an on-policy manner, where the student generates its own training data. To avoid disrupting the student’s distribution during early training, we sampled the initial sampling noise from $\mathcal{N}(0, 0.1)$.

\begin{table}[h]
\caption{\textbf{Hyperparameter Details for MiniLLM training.}}
\label{tab:hyperparameter_MiniLLM}
\centering
\begin{tabular}{lc}
\toprule
\textbf{Hyperparameter} & \textbf{Value} \\
\midrule\midrule
Conversational prompt & True \\
Max completion length & 400 \\
$Num_\text{generations}$ & 4 \\
Temperature & 1.0 \\
$Top_\text{p}$ & 1.0 \\
Batch size & 8 \\
Gradient accumulation steps & 2 \\
Epochs & 1 \\
Weight decay & 0 \\
Warmup steps & 150 \\
Lr Scheduler & Cosine \\
Single step decomposition & True \\
$Lora_\text{r}$ & 16 \\
$Lora_\alpha$ & 32 \\
$Lora_\text{dropout}$ & 0.1 \\
Lr & 7e-5 \\
Lora modules & $q_\text{proj}$, $k_\text{proj}$, $v_\text{proj}$, $o_\text{proj}$, $gate_\text{proj}$, $up_\text{proj}$, $down_\text{proj}$ \\
\bottomrule
\end{tabular}
\end{table}

\clearpage
\subsection{Details of Hallucination Detection Test}
For the GSM8K test set, we sample one answer per question from the target model and repeat this process three times. During sampling, we instruct the model to place the final answer inside \texttt{\textbackslash boxed\{\}}, and then assign a binary correctness label (0 or 1) by comparing the extracted boxed answer to the GSM8K ground-truth solution. This procedure yields three datasets.

For each dataset, we compute token-level epistemic uncertainty (EU) at every token position in every generated answer sequence, and aggregate these token-level values to obtain a sequence-level EU score. Since EU is not a probability, we pass the sequence-level EU through a logistic function to map it to $[0,1]$. We train this logistic mapping on a training split of the sequence-level EU values. When applied to the test split, the resulting value in $[0,1]$ is interpreted as the probability that the generated answer is incorrect. Finally, we compute AUROC, ECE, and Brier score using these probabilities and the correctness labels of the sampled answers.
 
\subsection{Details of Ensemble Configuration}
\label{app:detail_ensemble}
In this study, we implement Expected Uncertainty (EU) estimation using two distinct approaches based on how the perturbation distribution is constructed.

\paragraph{Approach 1: Ensemble of Trained Draft Models ($S \times M$).}
The first approach constructs the estimator by ensembling explicitly trained draft models. We define a configuration $S \times M$, where $S$ denotes the number of disjoint data partitions (sub-ensembles) and $M$ denotes the number of models trained per partition.
We trained a total pool of nine draft models. Specifically, three draft models form a single ensemble unit (trained on one partition), and we utilize up to three such units. From this pool, we construct the following configurations:
\begin{itemize}
    \item \textbf{1 $\times$ 3}: Selects 3 models from a single partition (Total 3 models).
    \item \textbf{2 $\times$ 3}: Selects 3 models each from 2 disjoint partitions (Total 6 models).
    \item \textbf{3 $\times$ 3}: Selects 3 models each from 3 disjoint partitions (Total 9 models).
\end{itemize}
The final EU is computed based on the disagreement across these aggregated trained models.

\paragraph{Approach 2: Single Draft with Noise Injection ($K$ only).}
The second approach utilizes only a single draft model selected from the trained pool. Instead of combining distinct model weights, we generate "draft perturbations" at inference time by injecting low-rank noise sampled from $\mathcal{N}(0, 0.1)$ into the single draft model. We denote these settings as "$K$ only" (e.g., 3 only, 6 only), where EU is computed directly from the disagreement across $K$ noise-sampled forward passes.

\clearpage

\section{Additional Results}
\label{app:additional_results}

\subsection{Performance Comparison under Fixed Budget ($K=6$)}
We further evaluate the uncertainty estimation performance under a fixed computational budget of $K=6$. As shown in Tables \ref{tab:ue_8b_3b_6} and \ref{tab:ue_1b_6}, DDD (Ours) achieves the best performance, ranking first among all methods in this setting. 
Specifically, in the 8B $\rightarrow$ 3B scenario, DDD records the lowest RMSE of 0.2145, outperforming all baselines.

This superiority is maintained across different model scales. Even with smaller 1B draft models (Table \ref{tab:ue_1b_6}), DDD consistently secures the top spot, significantly surpassing the 6 only approach. These results confirm that under the specific constraint of $K=6$, our proposed DDD strategy is the optimal choice for maximizing estimation fidelity.

\begin{table}[h]
\caption{\textbf{Uncertainty Estimation Performance (8B $\rightarrow$ 3B).} Comparison of RMSE ($\downarrow$) and Spearman Correlation ($\uparrow$) across different methods.}
\label{tab:ue_8b_3b_6}
\centering
\begin{tabular}{l|c|cc}
\toprule
\textbf{Method} & \textbf{Setting} & \textbf{RMSE} $\downarrow$ & \textbf{Spearman} $\uparrow$ \\
\midrule\midrule
Baseline(No Train) & 6 & 0.3315 & 0.8810 \\
GKD & 6 & 0.3245 & 0.8958 \\
MiniLLM & 6 & 0.3121 & 0.8991 \\
IDD & 6 & 0.3680 & 0.9139 \\
\textbf{DDD (Ours)} & \textbf{6} & \textbf{0.2145} & \textbf{0.9181} \\
\bottomrule
\end{tabular}
\end{table}

\begin{table*}[h]
\caption{\textbf{Uncertainty Estimation Performance (Smaller Draft Models).} Comparison of RMSE ($\downarrow$) and Spearman Correlation ($\uparrow$) for 1B draft models. Our \textbf{DDD (6)} strategy consistently outperforms baselines in capturing the target's uncertainty.}
\label{tab:ue_1b_6}
\centering
\resizebox{0.7\textwidth}{!}{
\begin{tabular}{l|c|cc|cc}
\toprule
\multicolumn{1}{c}{}&\multicolumn{1}{c}{}& \multicolumn{2}{c}{\textbf{8B (Target) $\rightarrow$ 1B (Draft)}} & \multicolumn{2}{c}{\textbf{3B (Target) $\rightarrow$ 1B (Draft)}} \\
\cmidrule(lr){3-4}\cmidrule(lr){5-6}
\multicolumn{1}{c}{\textbf{Method}} & \multicolumn{1}{c}{\textbf{Setting}} & \textbf{RMSE} $\downarrow$ & \multicolumn{1}{c}{\textbf{Spearman} $\uparrow$} & \textbf{RMSE} $\downarrow$ & \textbf{Spearman} $\uparrow$ \\
\midrule\midrule
Baseline(No Train) & 6 & 0.5865 & 0.8264 & 0.4733 & 0.8451 \\
IDD & 6 & 0.7705 & 0.8827 & 0.6653 & 0.8941 \\
\textbf{DDD (Ours)} & \textbf{6} & \textbf{0.4696} & \textbf{0.8905} & \textbf{0.4227} & \textbf{0.9024} \\
\bottomrule
\end{tabular}
}
\end{table*}

\clearpage

\subsection{Robustness of Proxy Model on Smaller Target (3B)}
We further validate the robustness of the $p_\text{mix}$ proxy using a smaller target model (Llama-3.2-3B-Instruct). As shown in Table \ref{tab:pmix_robustness_3b}, the 3B $p_\text{mix}$ proxy exhibits significantly lower RMSE and higher Spearman correlation compared to the raw target model under varying perturbations ($K=3, 10$). Figure \ref{fig:pmix_stability_boxplot} illustrates that the variance of the estimation error is markedly reduced for $p_\text{mix}$. This confirms that our distillation process effectively filters aleatoric noise, providing a stable ground truth regardless of the target model's size.

\begin{table}[h]
\caption{\textbf{Robustness of 3B $p_\text{mix}$ Proxy.} Comparison of RMSE ($\downarrow$) and Spearman correlation ($\uparrow$) between the raw target model and our distilled $p_\text{mix}$ proxy under varying perturbations.}
\label{tab:pmix_robustness_3b}
\centering
\resizebox{0.5\textwidth}{!}{
\begin{tabular}{l|l|cc}
\toprule
\textbf{Perturbation} & \textbf{Metric} & \textbf{Target (raw)} & \textbf{$p_\text{mix}$ (Ours)} \\
\midrule\midrule
\multirow{2}{*}{\textbf{K=3}} 
& RMSE ($\downarrow$) & 0.0274 & \textbf{0.0181} \\
& Spearman ($\uparrow$) & \textbf{0.9940} & 0.9846 \\
\midrule
\multirow{2}{*}{\textbf{K=10}}
& RMSE ($\downarrow$) & 0.0194 & \textbf{0.0104} \\
& Spearman ($\uparrow$) & 0.9958 & \textbf{0.9974} \\
\bottomrule
\end{tabular}
}
\end{table}

\begin{figure}[h]
  \centering
  \includegraphics[width=0.5\columnwidth]{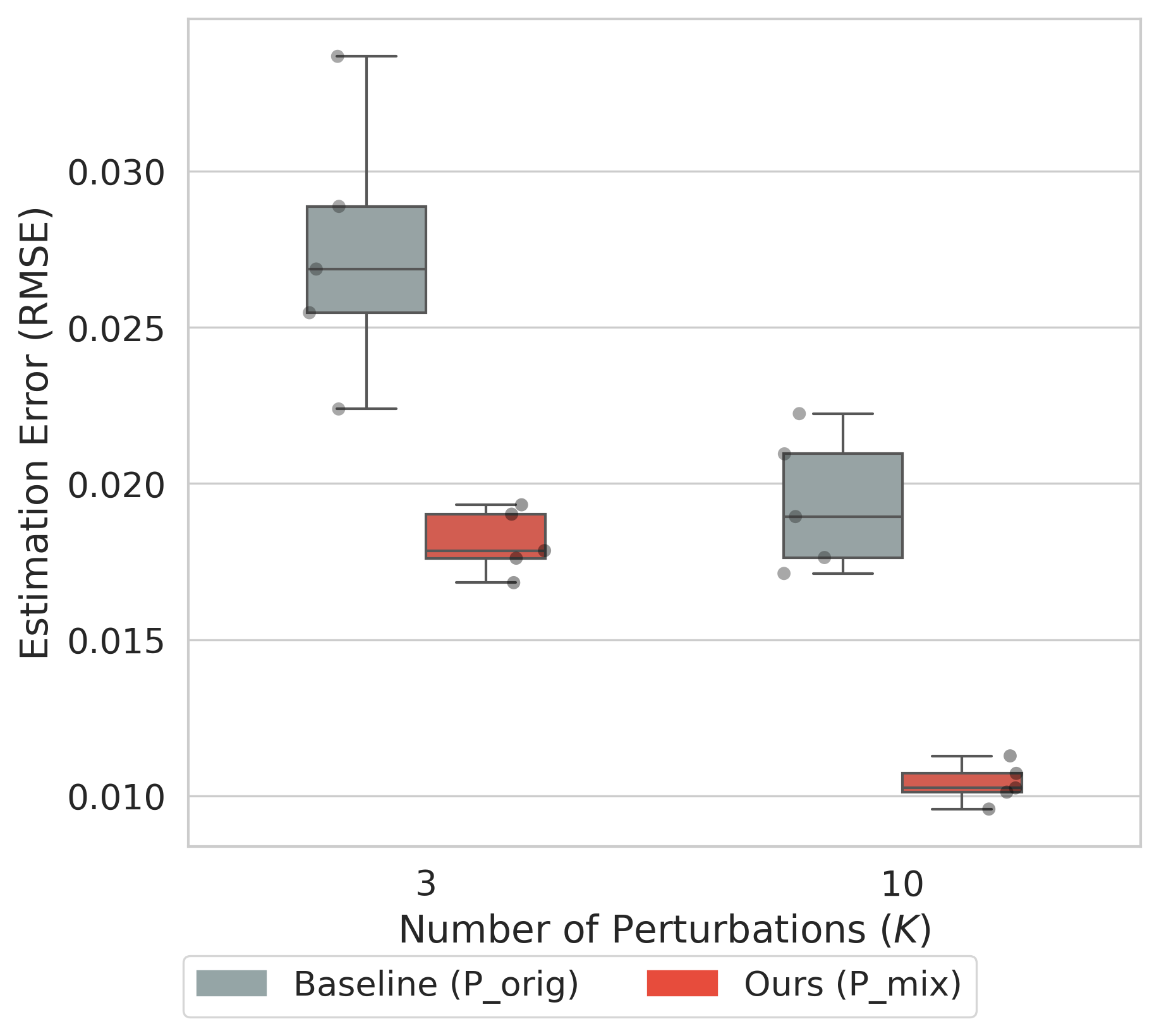}
  \caption{\textbf{Stability analysis for Llama-3.2-3B-Instruct comparing $p_{\text{orig}}$ (baseline) and $p_{\text{mix}}$ (Ours).} Box plots show RMSE distribution over 5 runs under $K=3$ and $K=10$ perturbations. Lower RMSE indicates more stable estimation.}
  \label{fig:pmix_stability_boxplot}
\end{figure}

\clearpage
\subsection{Impact of Data Partitioning on 1B Draft Models}
Finally, we investigate the impact of data partitioning strategies specifically on 1B draft models, extending the analysis from the main text. Tables \ref{tab:data_splits_8b1b} and \ref{tab:data_splits_3b1b} compare disjoint split configurations (e.g., $1\times3$, $2\times3$) and pooled baselines (e.g., "6 only") for the 8B $\rightarrow$ 1B and 3B $\rightarrow$ 1B tasks.

The results consistently show that data-diverse configurations (such as $2\times3$) achieve lower RMSE than training a larger ensemble on a single pooled dataset. This indicates that the benefit of data partitioning is not limited to larger models but is equally critical—if not more so—for lightweight draft models (1B) to prevent mode collapse and improve ensemble calibration.

\begin{table}[h]
\caption{\textbf{Impact of Data Partitioning.} Evaluation of different split configurations on 8B$\rightarrow$1B RMSE ($\downarrow$). Data-diverse partitioning consistently reduces estimation error.}
\label{tab:data_splits_8b1b}
\centering
\begin{tabular}{lc}
\toprule
\textbf{Split Config} & \textbf{RMSE} ($\downarrow$) \\
\midrule\midrule
$1\times3$ & \textbf{0.4388} \\
\textbf{$2\times3$ (Ours)} & 0.4437 \\
$3\times3$ & 0.4478 \\
3 only & 0.4771 \\
6 only & 0.4696 \\
9 only & 0.4674 \\
\bottomrule
\end{tabular}
\end{table}

\begin{table}[h]
\caption{\textbf{Impact of Data Partitioning.} Evaluation of different split configurations on 3B$\rightarrow$1B RMSE ($\downarrow$). Data-diverse partitioning consistently reduces estimation error.}
\label{tab:data_splits_3b1b}
\centering
\begin{tabular}{lc}
\toprule
\textbf{Split Config} & \textbf{RMSE} ($\downarrow$) \\
\midrule\midrule
$1\times3$ & 0.4022 \\
\textbf{$2\times3$ (Ours)} & 0.4023 \\
$3\times3$ & \textbf{0.3974} \\
3 only & 0.4238 \\
6 only & 0.4227 \\
9 only & 0.4205 \\
\bottomrule
\end{tabular}
\end{table}

\end{document}